\documentclass[letterpaper]{article} % DO NOT CHANGE THIS
\usepackage{aaai24}  % DO NOT CHANGE THIS
\usepackage{times}  % DO NOT CHANGE THIS
\usepackage{helvet}  % DO NOT CHANGE THIS
\usepackage{courier}  % DO NOT CHANGE THIS
\usepackage[hyphens]{url}  % DO NOT CHANGE THIS
\usepackage{graphicx} % DO NOT CHANGE THIS
\usepackage{mathtools,amssymb, amsthm, amsfonts}
\usepackage{comment}
\usepackage{xcolor}
\usepackage{pgfplots}
\usepackage{natbib}  % DO NOT CHANGE THIS AND DO NOT ADD ANY OPTIONS TO IT
\usepackage{caption} % DO NOT CHANGE THIS AND DO NOT ADD ANY OPTIONS TO IT
\usepackage{algorithm}
\usepackage{algorithmic}
\usepackage{cleveref}

\usepackage{newfloat}
\usepackage{listings}
\DeclareCaptionStyle{ruled}{labelfont=normalfont,labelsep=colon,strut=off} % DO NOT CHANGE THIS
\floatstyle{ruled}
\newfloat{listing}{tb}{lst}{}
\floatname{listing}{Listing}
\DeclareMathOperator*{\bE}{{\mathbb{E}}}
\DeclareMathOperator{\Reg}{\operatorname{Reg}}

\DeclareMathOperator{\A}{\mathcal{A}}
\DeclareMathOperator{\Q}{\mathcal{Q}}
\newcommand{\abs}[1]{\left|{#1}\right|}

\DeclareMathOperator{\E}{\mathcal{E}}

\DeclareMathOperator{\cali}{\mathcal{I}}

\DeclareMathOperator{\F}{\mathcal{F}}
\DeclareMathOperator{\G}{\mathcal{G}}

\newcommand{\eps}{\epsilon}

\newcommand{\supp}{\mathtt{supp}}
\newcommand{\Pk}{\text{Proj}_k}
\newcommand{\Lk}{\text{Last}_k}

\theoremstyle{plain}
\newtheorem{theorem}{Theorem}

\newtheorem{fact}[theorem]{Fact}
\newtheorem{claim}[theorem]{Claim}
\newtheorem{lemma}[theorem]{Lemma}
\newtheorem{observation}[theorem]{Observation}
\newtheorem{corollary}[theorem]{Corollary}
\theoremstyle{definition}
\newtheorem{definition}[theorem]{Definition}

\theoremstyle{remark}

\excludecomment{appendixcontent} %remove to compile the whole version

\title{Communication-Efficient Collaborative Regret Minimization\\  in Multi-Armed Bandits}
\author{
	Nikolai Karpov,  %\equalcontrib,
	Qin Zhang%\equalcontrib
}
\affiliations{
	Indiana University \\
	 Bloomington, IN 47405, USA \\
	nkarpov@iu.edu, qzhangcs@indiana.edu
}

\begin{document}

\maketitle

\begin{abstract}
    In this paper, we study the collaborative learning model, which concerns the tradeoff between {\em parallelism} and {\em communication overhead} in multi-agent multi-armed bandits.  For regret minimization in multi-armed bandits, we present the first set of tradeoffs between the number of rounds of communication among the agents and the regret of the collaborative learning process.
\end{abstract}

\section{Introduction}
\label{sec:introduction}

One of the biggest challenges with machine learning is {\em scalability}.   In recent years, a series of papers \citep{TZZ19,KZZ20,WHCW20,KZ22,KZ22b} studied bandit problems in the collaborative learning (CL) model, where multiple agents interact with the environment to learn simultaneously and cooperatively.  One of the most expensive resources in the CL model is {\em communication}, which consists of the number of communication steps (round complexity) and the total bits of messages exchanged between agents (bit complexity).  Communication directly contributes to the learning time due to network bandwidth constraints and latency, and it can also lead to significant energy consumption, especially for deep-sea or outer-space exploration tasks. Moreover, when messages are sent using mobile devices, communication can result in significant data usage.  In this paper, we focus on the round complexity in the CL model and consider a basic problem in the bandit theory named {\em regret minimization in multi-armed bandits} (MAB for short). We try to investigate the round-regret tradeoffs for MAB in the CL model.

In the rest of this section, we will first introduce the CL model and the MAB problem.  We then describe our results and place them within the context of the literature. 

\paragraph{Regret Minimization in MAB.}
In the single-agent learning model,  we have one agent and a set of arms $I = \{1, 2, \ldots, N\}$; the arm $i$ is associated with a distribution $\mathcal{D}_i$ with support $[0, 1]$ and (unknown) mean $\mu_i$.  At each time step $t = 1, 2, \ldots, T$, the agent pulls arm ${\pi_t}$ and receives a reward $r_t$ from distribution \(\mathcal{D}_{\pi_t}\).  The expected regret of a $T$-time single-agent algorithm $\A$ on input $I$ is defined to be
\begin{equation}
	\label{eq:regret-1}
	\bE[\Reg(\A(I, T))] = \bE\left[\sum_{t \in [T]} \left(\mu_{\star} - \mu_{\pi_t}\right)\right],
\end{equation}
where $\mu_{\star} \triangleq \max\limits_{i \in [N]} \{ \mu_{i} \}$.\footnote{We use $[n]$ to denote $\{1, 2, \ldots, n\}$.}   Without loss of generality, we assume that the best arm is unique.

\paragraph{The Collaborative Learning Model.}
The CL model was formalized in \citet{TZZ19}. In this model, we have $K$ agents and a set of $N$ arms $I = \{1, 2, \ldots. N\}$, where  arm $i$ has mean $\mu_i$.  Again let $\mu_{\star} \triangleq \max_{i \in [N]} \{ \mu_{i} \}$.
The learning proceeds in rounds.  Within each round, at each time step $t$, each agent $k\ (k \in [K])$ pulls arm ${\pi_t^{(k)}}$ based on its previous pull outcomes and messages received from other agents; the arms $\{{\pi_t^{(k)}}\}_{k \in [K]}$ can be the same or different for different agents.  At the end of each round, the $K$ agents communicate with each other to exchange newly observed information and determine the number of time steps for the next round. The number of time steps for the first round is fixed at the beginning.  See Figure~\ref{fig:CL} for an illustration of the CL model.

\begin{figure}[t] 
	\centering
	\includegraphics[width=0.45\textwidth]{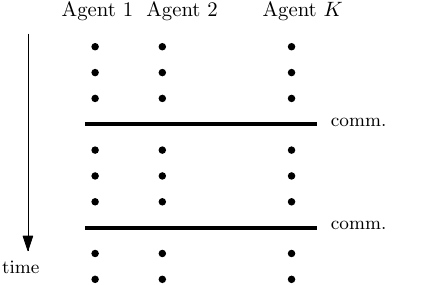} 
	\caption{The collaborative learning model. Each dot represents an arm pull.  The number of rounds of the learning process shown in the figure is $3$.}
	\label{fig:CL} 
\end{figure}

It is worth mentioning that the lengths of rounds are {\em not} required to be determined beforehand in the CL model. Though for most CL algorithms in the literature (including the one proposed in this paper), the round lengths are indeed fixed at the beginning of the algorithms. This relaxation will only make the lower bound proof harder/stronger.

The expected regret of a $T$-time $K$-agent collaborative algorithm $\A_K$ for MAB on input $I$ is defined to be
\begin{equation}
	\label{eq:regret-K}
	\bE\left[\Reg(\A_K(I, T))\right] = \bE\left[\sum_{t \in [T]} \sum_{k \in [K]} \left(\mu_{\star} - \mu_{\pi_t^{(k)}}\right)\right].
\end{equation}

\paragraph{The Batched Learning Model.}  
The CL model is closely related to the batched learning model, which has recently received considerable attention in bandit theory~\citep{PRCS15,JJNZ16,AAAK17,JSXC19,GHRZ19, EKMM19,BXJW19,KZ20,JTX+21}. 

In the batched model, there is one agent interacting with the arms.  The learning proceeds in batches.  The sequence of arm pulls in each batch need to be determined at the beginning of the batch. The goal is for the agent to minimize the regret over a sequence of $T$ pulls using a small number of batches.  

The batched model is motivated by applications in which there is a significant delay in getting back the observations, such as clinical trials~\citep{Thompson33,Robbins52} and crowdsourcing~\citep{KCS08}.  

The following observation connects the CL model and the batched model.
\begin{observation}
\label{ob:reduction}

If there is a $T$-time $R$-batch single-agent algorithm that achieves an expected regret $f(I)$ for any input $I$, then there is a $({T}/{K})$-time $R$-round $K$-agent collaborative algorithm that achieves an expected regret $f(I)$ for any input $I$.
\end{observation}
To see this, just note that each round of $z$ (non-adaptive) pulls in a batched algorithm can be evenly distributed to the $K$ agents in a collaborative algorithm so that each agent makes $z/K$ non-adaptive pulls.
%\footnote{We ignore floor/ceiling operators in this paper since they will not affect any asymptotic upper/lower bounds.}  
Observation~\ref{ob:reduction} allows us to establish a lower bound in the batched model by proving a corresponding lower bound in the CL model, and to design an algorithm for the CL model using an algorithm for the batched model.  

It is important to note that Observation~\ref{ob:reduction} is {\em one-way}; the other direction does {\em not} hold. This is because the CL model is {\bf strictly stronger} than the batched model in the sense that in the CL model, each agent can make {\em adaptive} pulls within each round. While in the batched model, the sequence of pulls are {\em non-adaptive} in each batch.  The requirement to accommodate {\em local agent adaptivity} makes the previous approaches for proving lower bounds in the batched model inapplicable to the CL model.  
As an example, in the previous work~\citep{TZZ19}, it has been shown that for the problem of {\em best arm identification} (BAI) in multi-armed bandits, where the goal is to identify the arm with the highest mean rather than minimizing the regret, $O(\log K)$ rounds is sufficient to achieve almost optimal error probability under a time budget in the adaptive CL model.  On the other hand, the same paper shows that $\Omega(\log N/\log\log N)$ rounds is necessary to achieve almost optimal error probability under a time budget in the batched model.\footnote{In \citet{TZZ19}, the $\Omega(\log N/\log\log N)$ (recall that $N$ is the number of arms) round lower bound was proved for the {\em non-adaptive} CL model, which is equivalent to the batched model.} This shows that the local adaptivity does make the BAI problem more difficult in the CL model when $N \gg K$.

\paragraph{Our Results.}
Let $\star = \arg\max_{i \in [N]} \mu_i$ be the index of the best arm.  Let $\Delta_i \triangleq \mu_\star - \mu_i$, and $\Delta(I) = \min_{i \neq \star} \Delta_i$.   All logarithms have base $2$ unless specified explicitly.  For readability, we use `$\tilde{\ \ }$' to hide some logarithmic factors.  All these factors will be spelled out in the corresponding theorems and corollaries.

The results of this paper include the followings.
\begin{enumerate}
\item Our main result is a lower bound for MAB in the CL model (Theorem~\ref{thm:lb}).  We show that for any $T$-time $K$-agent collaborative algorithm $\A_K$, there is an input $I$ such that if $\A_K$ runs on $I$ using at most $R \le \frac{\log(KT)}{2\log\log\log(KT)}$ rounds, then $\A_K$ incurs an expected regret of $\tilde{\Omega}\left(\min\{K, (KT)^{\frac{1}{R}}\} \cdot \frac{1}{\Delta(I)}\right)$.   
%The lower bound result holds even when there are only two arms. 

\item Using Observation~\ref{ob:reduction}, our lower bound for MAB in the CL model also gives a lower bound for MAB in the batched model (Corollary~\ref{cor:lb}), which is comparable to the previous best lower bound~\cite{GHRZ19}.

\item We also design an algorithm for batched MAB (Theorem~\ref{thm:ub}). 
Again via Observation~\ref{ob:reduction}, we obtain an algorithm for MAB in the CL model.  Our upper bound matches the lower bound up to logarithmic factors in regret.
\end{enumerate}

We note that there is a single-agent algorithm $\A^{\tt BPR}$ \citep{BPR13} such that for an input of two arms, given a time budget $T$, the algorithm incurs a regret of $\tilde{O}\left( \frac{1}{\Delta(I)}\right)$. Therefore, in the multi-agent setting by asking each of the $K$ agents to run $\A^{\tt BPR}$, we obtain a multi-agent algorithm $\A^{\tt BPR}_K$ that incurs a regret of $\tilde{O}\left(K \cdot \frac{1}{\Delta(I)}\right)$ {\em without} using any communication.  This is why the $\min\{K, (KT)^{\frac{1}{R}}\}$ term in our lower bound is necessary.

On the other hand, we observe that when $R \le \frac{\log(KT)}{\log K}$, any collaborative algorithm $\A_K$, if runs for $T$ time steps, incurs an expected regret of at least $\tilde{\Omega}\left( K \cdot \frac{1}{\Delta(I)}\right)$, which matches the upper bound given by $\A^{\tt BPR}_K$ up to logarithmic factors.  Therefore, Theorem~\ref{thm:lb} indicates that to achieve any super-logarithmic reduction in regret for MAB in the CL model, the agents need to use at least $\frac{\log(KT)}{\log K}$ rounds of communication.

To the best of our knowledge, our proof strategy for the round lower bound in the CL model is new.  The only previous technique for proving round lower bound in the CL model is the {\em generalized round elimination}, which was proposed in \citet{TZZ19} for the problem of best arm identification.  However, we found it difficult to use this technique for proving a regret-round tradeoff, primarily because of  the different nature of hard input instances between best arm identification and regret minimization. Specifically,  \citet{TZZ19} exploited a pyramid-type construction on arms for best arm identification, while our hard inputs for regret minimization only involve two arms. If we directly apply round elimination to our hard inputs, only the best arm would survive after one elimination step. Consequently, the maximum round lower bound we can prove using round elimination is only two.

\section{Related Work}
\label{sec:related}

\paragraph{Work in the Collaborative Learning Model.}
To the best of our knowledge, the study of the CL model began from \citet{HKK+13}, in which the authors considered the problem of best arm identification (BAI) in MAB.  However, \citet{HKK+13} only considered a special case for the lower bound, and the CL model was {\em not} formally defined in their paper.

The CL model that we use in this paper was introduced by \citet{TZZ19}, in which almost tight round-time tradeoff was given for BAI in MAB.\footnote{In \citet{TZZ19}, the time cost was presented as {\em speedup}, defined as the ratio between the running time of the collaborative algorithm and that of the best centralized algorithm.}   \citet{KZZ20} extended this line of work to the top-$m$ arm identification in MAB.   The work of \citet{KZ22} investigated the bit complexity of BAI in the CL model.  \citet{WHCW20} studied regret minimization in MAB in the same model, but their primary focus is the bit complexity.  \citet{DSV+23} investigated neural contextual bandits; their focus was only on the upper bounds.

Several recent papers~\citep{SS21,SSY21,KZ22b} studied problems in MAB in the {\em non-IID} CL model, where agents interact with possibly different environments.  More specifically, \citet{SS21,SSY21} studied regret minimization with a focus on the bit complexity, but the bit cost in their model is integrated into the regret formulation.  \citet{KZ22b} gave almost tight round-time tradeoff for BAI in the non-IID CL model. \citet{RVK21} gave collaborative algorithms for non-IID BAI and regret minimization in MAB in a similar setting, but their algorithmic results only consider the fixed-confidence setting, and they did not give any lower bound on the round-regret tradeoffs.

\paragraph{Work in the Batched Learning Model.}  Batched algorithms for bandit problems have attracted significant attention in the past decade. As discussed previously, \citet{GHRZ19,EKMM19} studied regret minimization in MAB mentioned.  An earlier work \citep{PRCS15} studied the same problem on two arms.  \citet{JTX+21} considered asymptotic regret in MAB in the batched model. Several recent papers studied batched regret minimization in MAB using Thompson sampling~\citep{KO21,KMS21,KZ21}.  Another series of works \citep{JJNZ16,AAAK17,JSXC19} studied  batched BAI in MAB.  

\paragraph{Other Work in Multi-Agent Bandit Learning.}
There are many other papers investigating multi-agent bandit learning, but they do not focus on the round complexity of the learning process. A series of papers \citep{SBH+13,LSL16,LSL18} considered MAB problems in the peer-to-peer (P2P) computing models such that at each time step, agents can only communicate with their neighbors in the P2P network.  Several papers \citep{LZ10,RSS16,BL18,BB20} considered the collision model, in which if multiple agents try to pull the same arm at a particular time step, then their rewards will be reduced due to collision.  

There is a line of research that studies the bit complexity of the messages transmitted between the agents~\citep{ML21, CSGS20, WYCLHTL23, WPAJR20,WHCW20,HWYS21, HWMG22,LWWW22}.  Some recent work has extended this line of research to related models such as the Markov Decision Processes~\citep{DP21,MHWG23}.

\section{The Lower Bound}
\label{sec:lb}

In this section, we give the following theorem, which is the main result of this paper.  

\begin{theorem}
	\label{thm:lb}
	For any $R$ such that $1 \le R \le \frac{\log(KT)}{2\log\log\log(KT)}$, and for any $R$-round $T$-time $K$-agent collaborative algorithm $\A_K$ for MAB, there is an input $I$ such that $\A_K$ incurs an expected regret of $\Omega\left(\frac{1}{{\log(KT)\log\log\log(KT)}} \cdot \min\left\{K, (KT)^{\frac{1}{R}}\right\}\cdot \frac{1}{\Delta(I)}\right)$ on input $I$.
\end{theorem}

By Observation~\ref{ob:reduction}, we have the following corollary (we chose a value $K$ such that $K =  (KT)^{\frac{1}{R}}$ in Theorem~\ref{thm:lb}, and note that a time budget $T/K$ in the CL model corresponds to a time budget $T$ in the batched model).
\begin{corollary}
	\label{cor:lb}
	For any $R$ such that $1 \le R \le \frac{\log T}{2\log\log\log T}$, and for any $R$-round $T$-time batch algorithm $\A$ for MAB, there is an input $I$ such that $\A$ incurs an expected regret of $\Omega\left(\frac{1}{{\log T \log\log\log T}} \cdot T^{\frac{1}{R}} \cdot \frac{1}{\Delta(I)}\right)$ on input $I$.
\end{corollary}

This result is comparable with the lower bound result for (adaptive grid) batched algorithms in \citet{GHRZ19}. In particular, both results show that $\Omega(\log T/\log \log T)$ rounds is necessary   to achieve the optimal regret $O\left(\log T \cdot \frac{1}{\Delta(I)}\right)$. 

In the rest of this section we prove Theorem~\ref{thm:lb}.

\subsection{The Setup}
We start by introducing some concepts and notations.

\paragraph{Notations.}  
We list in Table~\ref{tab:notation} a set of key notations that we will use throughout this paper.   Readers can always come back to this table when encounter an unfamiliar notation.

\begin{table}[t]
	\centering
		\begin{tabular}{|p{.06\textwidth}| p{.37\textwidth}|}
			\hline
			Notation & Definition\\
			\hline
			$N$ &  number of arms \\
			\hline
			$K$ &  number of agents \\
			\hline
			$R$ &  number of rounds \\
			\hline
			$T$ &  time horizon \\
			\hline
			$\eps$, $\lambda$, $\beta$ & fixed constants: $\eps = 0.1$,  $\lambda =10^{-6}$, and $\beta = 4$\\
			\hline
			$\alpha$ & $\alpha \triangleq {\log L}/{(2\lambda)}$\\
			\hline
			$L$ & $L = \frac{\log(4KT)}{4}$ is the number of pairs of hard inputs \\
			\hline
			$\Delta_\ell$ & $\Delta_\ell = 2/\beta^\ell$ is the mean gap of two arms in the level $\ell$ hard inputs \\
			\hline
			$\gamma$ & pull transcript; a sequence of (arm pull index, reward) pairs\\
			\hline
		$r(\gamma)$ & see Definition~\ref{def:r-and-ell}; intuitively, it is the index of a ``big" round under transcript $\gamma$ \\
			\hline
			$\ell(\gamma)$ & see Definition~\ref{def:r-and-ell}; it is roughly the logarithm of the time step of the beginning of the $r(\gamma)$-th round \\
			\hline
			$\tau(\gamma, \ell)$ & see Definition~\ref{def:tau}; can be seen as a mapping from $\ell(\gamma)$ back to the round index \\
			\hline
$\ell^*$ &   defined in Inequality~(\ref{eq:h-4}) \\
			\hline 
		\end{tabular}
		\caption{Summary of Notations}
			\label{tab:notation}
\end{table}

We will use $R$ to denote the number of rounds used by the $K$-agent collaborative algorithm $\A_K$.  

For a time horizon $T$, we will create $L = \frac{\log(4KT)}{4}$ pairs of hard inputs, and focus on $R$ in the range
\begin{equation}
	\label{eq:R-L}
	\frac{4L}{\log K}  \le R \le \frac{2L}{\log\log L} .
\end{equation}
Note that in the case when $R < \frac{4L}{\log K}$, the regret will certainly be {\em lower} bounded by the case when $R  = \frac{4L}{\log K}$, in which case $(KT)^{\frac{1}{R}}$ becomes 
$$(KT)^{\frac{\log K}{4L}} = 2^{\log(KT) \cdot \frac{\log K}{\log(4KT)}} = \Theta(K).$$
This is why there is a $\min\{K, (KT)^{\frac{1}{R}}\}$ term inside the regret in Theorem~\ref{thm:lb}.  As mentioned in ``our results" in the introduction, this $\min$ operation is also necessary to be there.

We will use the following constants in the proof:  $\eps = 10^{-1}$, $\lambda =10^{-6}$, and $\beta = 4$.  We will use the  notations instead of the actual constants in most places of this section for the sake of readability.

\paragraph{Pull Transcript.}  Let $\gamma = ((j_1, o_1), \ldots, (j_n, o_n))$ be a sequence of pulls and reward outcomes on an input $I$ for MAB, where $j_t$ is the index of the arm in $I$ being pulled at the $t$-th time step and $o_t$ is the corresponding reward.  We call $\gamma$ the {\em transcript} of a sequence of arm pulls, and use $\abs{\gamma} = n$ to denote the length of $\gamma$ (i.e., the number of $(j_t, o_t)$ pairs in $\gamma$).  For convenience, we use $j(\gamma) = (j_1, \ldots, j_n)$ to denote the sequence of arm indices and $o(\gamma)$ to denote the corresponding sequence of rewards. 

For a sequence of arm indices $j(\gamma)$, let $\Theta_I(j(\gamma))$ be the sequence of (random) rewards by pulling the arms of $I$ according to $j(\gamma)$.  We define 
\begin{equation}
	\label{eq:g-def}
	g_I(\gamma) \triangleq \Pr[\Theta_I(j(\gamma)) = o(\gamma)],
\end{equation}
which is the probability of observing the reward sequence $o(\gamma)$ by pulling the arms in input $I$ following the index sequence $j(\gamma)$.

For a single-agent algorithm $\A$ for MAB, an input $I$ and a time horizon $n$, we use $\Gamma \sim \A(I, n)$ to denote a random transcript generated by running $\A$ on input $I$ for $n$ time steps. For a $K$-agent collaborative algorithm $\A_K$, we write $\Gamma \sim \A_K(I, n)$ as the {\em round-robin} concatenation of the $K$ transcripts generated by the $K$ agents on input $I$ for $n$ time steps.  That is,
$$
\Gamma = \left\{(J_1^{(1)}, O_1^{(1)}),  \ldots, (J_1^{(K)}, O_1^{(K)}), \ldots, (J_n^{(K)}, O_n^{(K)})\right\},
$$
where $(J_t^{(k)}, O_t^{(k)})$ is the (arm index, reward) pair of the pull of agent $k$ at time $t$. We use capital letters $J_t^{(k)}$ and $O_t^{(k)}$  since they are random variables depending on the previous pulls and outcomes.

\subsection{The Hard Inputs}
\label{sec:hard-dist}

We begin by introducing the set of hard inputs.

For each $\ell \in \{1,  \ldots, L\}$ and each $\sigma \in \{+1, -1\}$, let
$I_\ell^\sigma$ be an input on two Bernoulli arms (i.e., the reward is either $0$ or $1$ on each pull), where the first arm has mean $\frac{1}{2} + \frac{\sigma}{\beta^\ell}$ and the second arm has mean $\frac{1}{2} - \frac{\sigma}{\beta^\ell}$. 

For the convenience of writing, we will abbreviate $I_\ell^{+1}$ and $I_\ell^{-1}$ to  $I_\ell^+$ and $I_\ell^-$ respectively.

For $\ell \in [L]$, let $\Delta_\ell  = {2}/{\beta^\ell}$ be the mean gap between the two arms in the inputs $I_\ell^{+}$ (or $I_\ell^{-}$).  

We define the set of hard inputs to be
\begin{equation*}
	\cali = \{I^{+}_1, I^{-}_1, \dotsc, I^{+}_L, I^{-}_L\}\,.
\end{equation*}
Let $\cali_\ell = \{I^{+}_\ell, I^{-}_\ell, \dotsc, I^{+}_L, I^{-}_L\}$ denote a suffix of $\cali$ starting from index $\ell$.

The set of hard inputs $\cali$ have some nice properties which we will use in the lower bound proof. Due to the space constraints, we leave them to~\Cref{sec:input-property}.

\subsection{Indistinguishable Input Pairs}
\label{sec:typical}

We introduce the following event defined on a transcript $\gamma$.
\begin{definition} {\bf Event $\E(\gamma)$:}
	\label{def:E}
	For any $\ell \in [L]$ such that $\frac{\lambda  \beta^{2\ell}}{\log L}~\ge~\abs{\gamma}$, and for any pair of inputs ${A, B \in \cali_\ell}$, we have
	\begin{equation*}
		\label{eq:typical}
		\ln\frac{g_A(\gamma)}{g_B(\gamma)} \le  2\eps\ . 
	\end{equation*}
\end{definition}

Intuitively, it says that when the length of transcript $\gamma$ is smaller than $\frac{\lambda  \beta^{2\ell}}{\log L}$, the probabilities of producing $\gamma$ under all inputs in $\cali_\ell$ are similar. We will often abbreviate $\E(\gamma)$ to $\E$ when $\gamma$ is clear from the context.  

The following lemma states that for a random transcript $\Gamma$ generated by running a single-agent algorithm on any input in $\cali$, $\E(\Gamma)$ holds with high probability.  The technical proof can be found in \Cref{app:proof-lem-event-E}.

\begin{lemma}
	\label{lem:event-E}
	For any single-agent algorithm $\A$ for MAB, any $I \in \cali$, and any $n > 0$, let $\Gamma \sim \A(I, n)$ denote a random transcript $\Gamma$ generated by running $\A$ on input $I$ for $n$ time steps. It holds that $$\Pr_{\Gamma \sim \A(I, n)}[\E(\Gamma)] \ge 1 - {1}/{L^6}.$$
\end{lemma}

The next lemma shows that {\em short} transcripts generated by a single-agent algorithm on two inputs in $\cali_\ell$ are statistically indistinguishable.  Its proof can be found in \Cref{app:proof-lem-indistinguishable}.

\begin{lemma}
	\label{lem:indistinguishable}
	Let $\A$ be any single-agent algorithm for MAB. For a transcript $\gamma$, let $\G(\gamma)$ be any event determined by $\gamma$.  
	For any $\ell \in [L]$, any pair of inputs $A, B \in \cali_\ell$, and any $n \le \frac{\lambda \beta^{2\ell}}{ \log L}$, we have
	\begin{equation*}
		\label{eq:similar-1}
	 \Pr_{\Gamma \sim \A(A, n)}[\G(\Gamma) \wedge \E(\Gamma)] \le 		e^{2\eps} \Pr_{\Gamma \sim \A(B, n)}[\G(\Gamma) \wedge \E(\Gamma)]\ .
	\end{equation*}
\end{lemma}

\subsection{The Lower Bound Proof}
\label{sec:lb-proof}

Now we are ready to give the proof of Theorem~\ref{thm:lb}. 

\paragraph{Intuition.} 
The high level intuition is that if the number of rounds of the CL algorithm is small, then for some pair of inputs \((I^+_{\ell^*}, I^-_{\ell^*})\) in the set of hard inputs $\cali$, we have (1) the algorithm will make many pulls in the $\ell^*$-th round, and (2) the information collected from the pull transcript and previous communication {\em at each local agent} is not enough to distinguish \(I^+_{\ell^*}\) from \(I^-_{\ell^*}\). The second item implies that all sequences of pulls on the pair of inputs \((I^+_{\ell^*}, I^-_{\ell^*})\) are approximately equally likely, which, together with the first item, implies that the regret of the algorithm should be large on at least one of these two inputs. 

\paragraph{Identifying A Critical Pair of Inputs.}
We start by identifying the pair \((I^+_{\ell^*}, I^-_{\ell^*})\).

Observe that by our choices of $L$ and $\beta$, it holds that 
\begin{equation}
\label{eq:KT-L}
T = 1/(K\Delta_L^2) = \beta^{2L}/(4K).
\end{equation}   

Let $\A_K$ be a $R$-round collaborative algorithm.
Let $\gamma$ be any transcript produced by $\A_K$.  Let $t_r \triangleq t_r(\gamma)\ (r = 1, \ldots, R)$ be the time step at the end of the $r$-th round. We thus have $t_R = T$.  For convenience, we define $t_0 = 1/K$.  Note that $t_1, \ldots, t_{R-1}$ are determined by $\gamma$, and $t_0$ and $t_R$ are two fixed values.

We have the following simple fact on the ratio of finishing times of two consecutive rounds.
\begin{fact}
	\label{fact:time-partition}
	For any $T > 0$, $R > 0$, and any transcript $\gamma$, there is a $r \in [R]$ such that 
	$
	\frac{t_r}{t_{r-1}} \ge (KT)^{\frac{1}{R}}.
	$
\end{fact}

We define the following event $\F_r(\gamma)$ for $r = 1, \ldots, R$.
\begin{definition} {\bf Event $\F_r(\gamma)$:}
	\label{def:F}
	For any $i < r$, it holds that $t_i/t_{i - 1} < (KT)^\frac{1}{R}$; and for $i = r$, we have $t_r/t_{r - 1} \ge (KT)^\frac{1}{R}$.
\end{definition}
It is clear that $\F_1(\gamma), \ldots, \F_R(\gamma)$ are disjunctive and they together partition the probability space.

For convenience of writing, let $\alpha \triangleq \frac{\log L}{2\lambda}$.  We first introduce two notations $r(\gamma)$ and $\ell(\gamma)$; intuitively, the former is the index of a ``big" round under transcript $\gamma$, and the latter is roughly the logarithm of the time step of the beginning of the $r(\gamma)$-th round.

\begin{definition}[$r(\gamma)$ {\bf and} $\ell(\gamma)$]
	\label{def:r-and-ell}
	For a transcript $\gamma$, let $r = r(\gamma)$ be the round index such that $\F_r(\gamma)$ holds. And let $\ell(\gamma)$ be the integer such that
	\begin{equation}
		\label{eq:h-1}
		\frac{\beta^{2(\ell(\gamma)-1)}}{\alpha K } \le t_{{r(\gamma)}-1} < \frac{\beta^{2\ell(\gamma)}}{\alpha K }.
	\end{equation}
\end{definition}

The next claim shows that the value of $\ell(\gamma)$ will not be larger than $L$.  Its proof can be found in \Cref{app:proof-cla-ell-star}.

\begin{claim}
		For any $\gamma$, it holds that $1 \le \ell(\gamma) \le L$.
		\label{cla:ell-star}
\end{claim}

By the definition of $\F_r(\gamma)$, we have
\begin{equation}
	t_{r(\gamma)} \ge (KT)^{\frac{1}{R}} \cdot t_{r(\gamma)-1} = \left(\frac{\beta^{2L}}{4}\right)^{\frac{1}{R}} t_{r(\gamma)-1}\ .
	\label{eq:h-3}
\end{equation}

Let $m_{r(\gamma)} = t_{r(\gamma)} - t_{r(\gamma)-1}$ be the length of the $r(\gamma)$-th round. By (\ref{eq:h-1}) and (\ref{eq:h-3}), we have
\begin{eqnarray}
m_{r(\gamma)} \ge  \left(\left(\frac{\beta^{2L}}{4}\right)^{\frac{1}{R}}-1\right) \frac{\beta^{2(\ell(\gamma)-1)}}{\alpha K } .  \label{eq:h-5}
\end{eqnarray}

Now, consider a particular $\ell^*$ such that
\begin{equation}
		\label{eq:h-4}
	\Pr_{\Gamma \sim \A_K(I_L^{+}, T)}[\ell(\Gamma) = \ell^*] \ge \frac{1}{L} \ .
\end{equation}
Such an $\ell^*$ must exist, since each transcript $\Gamma = \gamma$ corresponds to a unique $r(\gamma)$ and consequently a unique $\ell(\gamma)$. And by Claim~\ref{cla:ell-star}, $\ell(\gamma) \le L$ always holds.

Our goal is show that the expected regret of $\A_K$ is high on either the input $I_{\ell^*}^+$ or the input $I_{\ell^*}^-$.  We call $(I_{\ell^*}^+, I_{\ell^*}^-)$ the {\em critical input pair} for $\A_K$.

\paragraph{Projection of A Collaborative Algorithm.}  To facilitate the regret analysis on the critical pair of inputs, we would like to introduce a concept termed as  {\em projection of a collaborative algorithm on a single agent}.  

We first introduce a notation $\tau(\gamma, \ell)$, which can be seen as a mapping from (the logarithm of) time step $\ell(\gamma)$ {\em back} to the round index. 
\begin{definition}[$\tau(\gamma, \ell)$]
	\label{def:tau}
	Let $\gamma$ be an arbitrary transcript generated by running $\A_K$ on $I$ for $T$ time steps. Let $\tau(\gamma, \ell)$ be the round index such that
	\begin{equation}
		\label{eq:tau}
		\frac{\beta^{2(\ell-1)}}{\alpha K } \le t_{{\tau(\gamma,\ell)}-1} < \frac{\beta^{2\ell}}{\alpha K }.
	\end{equation}
\end{definition}
By the Definition~\ref{def:r-and-ell} and Definition~\ref{def:tau}, it is not difficult to check that $r(\gamma) = \tau(\gamma, \ell(\gamma))$.

Let $\A_K$ be a collaborative algorithm.  For any $k \in [K]$, we use $\Pk^{\A_K}(I, \ell)$ to denote a {\em single-agent} algorithm that simulates $\A_K$.

And let 
\begin{equation}
	\label{eq:zeta}
	\zeta_\ell = \frac{\beta^{2\ell}}{\alpha} \cdot \frac{\beta^{2(\frac{L}{R}-1)}}{8K}.
\end{equation}

$\Pk^{\A_K}$ simulates $\A_K$ as follows:
In the first $(\tau(\gamma, \ell)-1)$ rounds, at each time step $t$, if agents $1, \ldots, K$ pull arms $a_t^{(1)}, \ldots, a_t^{(K)}$ in $I$ respectively under $\A_K$, then $\Pk^{\A_K}$ pulls arms $a_t^{(1)}, \ldots, a_t^{(K)}$ in $I$ in order.  In the $\tau(\gamma, \ell)$-th round, at each time step $t$ when $t \le \zeta_\ell + t_{\tau(\gamma, \ell)-1}$, if agent $k$ pulls arm $a_t^{(k)}$ in $I$ under $\A_K$, then $\Pk^{\A_K}$ also pulls arm $a_t^{(k)}$ in $I$.  

For a transcript $\gamma$ generated by running $\A_K$ and each $k \in [K]$, we introduce two concepts:
\begin{definition}[$\Pk(\gamma, \ell)$]
 Let $\Pk(\gamma, \ell)$ be the sequence of $(j_t, o_t)$ pairs in $\gamma$ generated by the $K$ agents in the round-robin fashion in the first $(\tau(\gamma, \ell) - 1)$ rounds, followed by  the first $\zeta_\ell$ of $(j_t, o_t)$ pairs in the $\tau(\gamma, \ell)$-th round (or until the end of the $\tau(\gamma, \ell)$-th round) in $\gamma$ generated by agent $k$.   
\end{definition}

$\Pk(\gamma, \ell)$ connects a pull transcript produced by a $K$-agent algorithm {\em in the eye of the $k$-th agent} at the time of the $\zeta_\ell$-th time step in the $\tau(\gamma, \ell)$-th round (or until the end of the $\tau(\gamma, \ell)$-th round)  with that produced  by a single-agent algorithm.
\medskip

\begin{definition}[$\Lk(\gamma, \ell)$]
Let $\Lk(\gamma, \ell)$ be the sequence of the first $\zeta_\ell$ of $(j_t, o_t)$ pairs in the $\tau(\gamma, \ell)$-th round (or until the end of the $\tau(\gamma, \ell)$-th round)  in $\gamma$ generated by agent $k$.  
\end{definition}

$\Lk(\gamma, \ell)$ can be seen as a suffix of $\Pk(\gamma, \ell)$ that is only observed locally at the $k$-agent in the $\tau(\gamma, \ell)$-th round.
\medskip

\paragraph{Large Regret on the Critical Input Pair.}
Now we are ready to lower bound the regret.
Let $\A_K$ be any $K$-agent collaborative algorithm, and $\ell^*$ satisfying Inequality (\ref{eq:h-4}).   We define the following event for a transcript $\gamma$.
\begin{definition}{\bf Event $\Q(\gamma)$:}
	\label{def:Q}
	 $\ell(\gamma) = \ell^*$.
\end{definition}
%We will abbreviate $\Q(\gamma)$ to $\Q(\gamma)$ when $\A_K$ is clear from the context. 
Inequality (\ref{eq:h-4}) implies that
\begin{equation}
	\label{eq:k-1}
	\Pr_{\Gamma \sim \A_K(I_L^{+}, T)}[\Q(\Gamma)] \ge \frac{1}{L} \ .
\end{equation}
Intuitively, Event $\Q(\gamma)$ says that the ``big" round under transcript $\gamma$ coincides to at least a $1/L$ fraction of transcripts produced by $\A_K$ on a particular input $I_L^{+}$. 

Let $\Gamma_L \sim \A_K(I_L^{+}, T)$.  Let
\[\Upsilon = \{\gamma \in \supp(\Gamma_L)\ |\ \Q(\gamma)\}, \] and for any $k \in [K]$, 
\begin{equation}
	\label{eq:Upsilon}
  \Upsilon_k(\ell^*) = \{\Pk(\gamma, \ell^*)\ |\ \gamma \in \Upsilon\}.
\end{equation}
We will try to show that $I_{\ell^*}^+, I_{\ell^*}^-$ are indistinguishable w.r.t.\ transcripts in $\Upsilon_k(\ell^*)$, and we will use the special input $I_L^+$ as a bridge.
Specifically, we show for each transcript $\gamma \in \Upsilon_k(\ell^*)$, the probability of producing $\gamma$ when the input instance is $I_{\ell^*}^+$ is close to the probability of producing $\gamma$ when the input instance is $I_{\ell^*}^-$\ .

Before doing this, we first upper bound the length of transcripts in $\Upsilon_k(\ell^*)$.
By (\ref{eq:zeta}) and (\ref{eq:h-5}), we have
\begin{equation}
	\label{eq:safe}
	\zeta_{\ell^*} \le \left(\left(\frac{\beta^{2L}}{4}\right)^{\frac{1}{R}}-1\right) \frac{\beta^{2(\ell^*-1)}}{\alpha K } \le m_{\tau(\gamma, \ell^*)} .
\end{equation}

Consequently, for any $\gamma \in \Upsilon_k(\ell^*)$,
\begin{eqnarray}
	\abs{\gamma} &=& K \cdot t_{\tau(\gamma, \ell^*)-1} + \zeta_{\ell^*}  \nonumber
	\\ &\le& K \cdot \frac{\beta^{2\ell^*}}{\alpha K } + \frac{\beta^{2\ell^*}}{\alpha} \cdot \frac{ \beta^{2(\frac{L}{R}-1)}}{8K} \nonumber  \\
	&\le& \frac{\lambda  \beta^{2\ell^*}}{ \log L} \ ,  \label{eq:k-2}
\end{eqnarray}
where the last inequality holds because $\beta^{2(\frac{L}{R}-1)} \le 8K$ by the first inequality in (\ref{eq:R-L}).

The following two claims exhibit properties of transcripts in $\Upsilon_k(\ell^*)$.  The first claim states that the probability of a random transcript being in $\Upsilon_k(\ell^*)$ is significant.   Its proof makes use of Lemma~\ref{lem:event-E} and Lemma~\ref{lem:indistinguishable}.

\begin{claim}
	\label{cla:large-mass}
	For any $I \in \{I_{\ell^*}^+, I_{\ell^*}^-\}$ and any $k \in [K]$, we have 
	$$
	\Pr_{\Gamma \sim {\A_K}(I, T)}[\Pk(\Gamma, \ell^*) \in \Upsilon_k(\ell^*)\ \wedge\ \E(\Pk(\Gamma, \ell^*))] \ge \frac{e^{-2\eps}}{2L}\ .$$
\end{claim}

The next claim states that it is difficult to use a transcript in $\Upsilon_k(\ell^*)$ to differentiate inputs $I_{\ell^*}^+$ (or $I_{\ell^*}^-$) from $I_L^+$.  Its proof makes use of Lemma~\ref{lem:indistinguishable}.  

\begin{claim}
	\label{cla:similar}
	For any $I \in \{I_{\ell^*}^+, I_{\ell^*}^-\}$ and any $k \in [K]$, for any $\gamma \in \Upsilon_k(\ell^*)$ such that $\E(\gamma)$ holds, we have 
\begin{eqnarray}
	&&\Pr_{\Gamma \sim {\A_K}(I, T)}[\Pk(\Gamma, \ell^*) = \gamma] \nonumber \\
&=& c_\eps \Pr_{\Gamma \sim \A_K(I_L^+, T)}[\Pk(\Gamma, \ell^*) = \gamma] \nonumber
\end{eqnarray}
	for some $c_\eps \in [e^{-2\eps}, e^{2\eps}]$.
\end{claim}

Recall that Lemma~\ref{lem:event-E} and Lemma~\ref{lem:indistinguishable} concern single-agent algorithms.  We use the following relation between a $K$-agent algorithm $\A_K$ and a single-agent algorithm $\A$ to connect Claim~\ref{cla:large-mass} and Claim~\ref{cla:similar} with Lemma~\ref{lem:event-E} and Lemma~\ref{lem:indistinguishable}:
\begin{equation*}
	\Pr_{\Gamma \sim \A_K(I_L^{+}, T)}[\Q(\Gamma)] = \Pr_{\Gamma \sim \Pk^{\A_K}(I_L^+, \ell^*)}[\Gamma \in \Upsilon_k(\ell^*)]  .
\end{equation*}
We leave the proofs of Claim~\ref{cla:large-mass} and Claim~\ref{cla:similar} to \Cref{app:cla-large-mass} and \Cref{app:cla-similar}, respectively.

\medskip
We now try to prove Theorem~\ref{thm:lb}.

Let $\Gamma^+ \sim \A_k(I_{\ell^*}^+, T)$, and $\Gamma^- \sim \A_k(I_{\ell^*}^-, T)$. By Claim~\ref{cla:similar} and Claim~\ref{cla:large-mass}, we know that 
\begin{eqnarray}
	&&\sum_{\gamma \in \Upsilon_k(\ell^*)} \min \left\{ \begin{aligned} &\Pr[\Pk(\Gamma^+, \ell^*) = \gamma], \\ &\Pr[\Pk(\Gamma^-, \ell^*) = \gamma] \end{aligned} \right\} \nonumber \\
	&\ge&  \frac{e^{-2\eps}}{2L} \cdot (c_\eps)^2 \ge  \frac{e^{-8\eps}}{2L} . \label{eq:i-1}
\end{eqnarray}

For an input $I$ and transcript $\gamma = ((j_1, o_1), \ldots, (j_{\abs{\gamma}}, o_{\abs{\gamma}}))$, let $\Reg(I, \gamma)$ denote the regret of pulling the arm sequence $j(\gamma)$ on the input $I$, that is, 
$
	\Reg(I, \gamma) = \sum_{t = 1, \ldots, \abs{\gamma}} (\mu_* - \mu_{j_t}).
$

For any transcript $\gamma \in \Upsilon_k(\ell^*)$ and any $k \in [K]$,  we consider the regret $U_k = \Reg\left(I_{\ell^*}^+, \Lk(\gamma, \ell^*)\right)$ and $V_k = \Reg\left(I_{\ell^*}^-, \Lk(\gamma, \ell^*)\right)$.  Due to our constructions of $I_{\ell^*}^+$ and $I_{\ell^*}^-$, we have for any $k \in [K]$,
\begin{equation}
\label{eq:z}
	U_k + V_k \ge  \Delta_{\ell^*} \cdot \zeta_{\ell^*}\ .
\end{equation}

Since for $k = 1, \ldots, K$, $\Lk(\gamma, \ell^*)$ are disjoint,  we have for any $\gamma \in \Upsilon$,
\begin{eqnarray}
	&& \Reg(I_{\ell^*}^+, \gamma) + \Reg(I_{\ell^*}^-, \gamma) \nonumber \\
	 &\ge&  \sum_{k \in [K]} (U_k + V_k)  \nonumber \\
	 &\stackrel{(\ref{eq:z})}{\ge}&  K \Delta_{\ell^*} \cdot \zeta_{\ell^*}  \nonumber \\
	 &=&  K \Delta_{\ell^*} \cdot   \frac{ \beta^{2\ell^*}}{\alpha } \cdot \frac{\beta^{2(\frac{L}{R}-1)}}{8 K} \nonumber \\
	&=& \frac{\beta^{2(\frac{L}{R}-1)} }{2\alpha} \cdot \frac{1}{\Delta_{\ell^*}} . \label{eq:j-1}
\end{eqnarray}

By (\ref{eq:i-1}) and (\ref{eq:j-1}), we have that 
\begin{eqnarray}
	 && \max\left\{\bE\left[\Reg(\A_K(I_{\ell^*}^+, T))\right], \bE\left[\Reg(\A_K(I_{\ell^*}^-, T))\right]\right\} \nonumber
	\\ &&\ge \frac{1}{2} \cdot \frac{e^{-8\eps}}{2L} \cdot \frac{\beta^{2(\frac{L}{R}-1)} }{2\alpha} \cdot \frac{1}{\Delta_{\ell^*}} \label{eq:reg} \nonumber \\
	&&= \Omega\left(\frac{\beta^{\frac{2L}{R}}}{L \log L} \cdot \frac{1}{\Delta_{\ell^*}}\right). \nonumber 
\end{eqnarray}
This concludes the proof of Theorem~\ref{thm:lb}.

\begin{algorithm}[t]
	\caption{\textsc{BatchedMAB}\((I, \lambda, T)\)}
	\label{alg:main}
	\begin{algorithmic}[1]
	\STATE Initialize a set of active arms \(I_0 \leftarrow I\)
	\STATE Set \(T_0 \leftarrow 0\)
	\FOR{\(i = 1, 2, \ldots, \log_\lambda T\)}
	\STATE Set \(T_{i} \leftarrow  \lambda^{i}\)
	\ENDFOR
	\STATE Set \(r \leftarrow \log_\lambda \log(T^3N)\)
	\STATE Pull each arm for \(T_{r-1}\) times \label{ln:init}
	\WHILE{\(r \le \log_\lambda T\) or \(\abs{I_r} > 1\)} \label{ln:trigger}
	\FOR{\(a \in I_r\)}
	\STATE Make \((T_r - T_{r - 1})\) pulls on arm \(a\)
	\STATE Compute \(\hat{\mu}^{r}_a\), the estimated mean after \(T_r\) pulls
	\ENDFOR
	\STATE Let \(\hat{\mu}^{r}_{\max} \leftarrow \max_{a \in I_r} \hat{\mu}^{r}_a\)
	\STATE Set \(I_{r+1} \leftarrow \left\{a \mid \hat{\mu}^{r}_{\max} - \hat{\mu}^{r}_{a} < 2\sqrt{\frac{\ln(T^3\abs{I})}{T_r}}\right\}\) \label{ln:pruning}
	\STATE Update \(r \leftarrow r + 1\)
	\ENDWHILE
	\IF{\(r < \log_\lambda T\)}
	\STATE Assign the rest of pulls to the single arm in \(I_r\)
	\ENDIF
	\end{algorithmic}
	\end{algorithm}
	
\section{The Algorithm}\label{sec:ub}
In this section, we design a batched algorithm for MAB, which  implies an algorithm for MAB in the CL model via Observation~\ref{ob:reduction}.  Our batched algorithm is described in Algorithm~\ref{alg:main}. It uses the successive elimination method. In each batch, we pull the remaining arms for an equal number of times and then eliminate those whose empirical means are smaller than the best one by a good margin. 

 Algorithm~\ref{alg:main} can be seen as a variant of the algorithm in~\citet{GHRZ19}. The main differences are: (1) Algorithm~\ref{alg:main} employs an early stopping rule (triggered when $\abs{I_r} = 1$ at Line~\ref{ln:trigger}), which leads to an instance-dependent batch complexity; and (2) it uses a preliminary exploration step (Line~\ref{ln:init}) to further reduce the number of batches.

The proof of the following theorem can be found in \Cref{app:ub-proof}.  
We note that Algorithm~\ref{alg:main} does {\em not} need to know $\Delta(I)$, but in the analysis we can upper bound both the number of batches and the regret in terms of $\Delta(I)$.

\begin{theorem}
	\label{thm:ub}
For any $\lambda \ge 2$, \textsc{BatchedMAB}\((I, \lambda, T)\) uses $\eta \le \log_\lambda T$ rounds and incurs an expected regret of  $O\left(\sum_{a \neq \star} \frac{\lambda \log T}{\Delta_a}\right)$.  We also have $\eta = O\left(\log_\lambda \frac{1}{\Delta(I)}\right)$ with probability $\left(1 - \frac{1}{T^3}\right)$.
\end{theorem}

By Observation~\ref{ob:reduction}, we have the following corollary.

\begin{corollary}
\label{cor:ub}
There is a collaborative algorithm $\A_K$ for MAB such that under time horizon $T$, for any input $I$, $\A_K$ uses $\eta \le  \log_\lambda (KT)$ rounds and incurs an expected regret of $O\left(\sum_{a \neq \star} \frac{\lambda \log(KT)}{\Delta_a}\right)$.  
We also have $\eta = O\left(\log_\lambda \frac{1 }{\Delta(I)}\right)$ with probability $\left(1 - \frac{1}{T^3}\right)$.
\end{corollary}

We would like to give a brief comparison between our upper bound and the lower bound.
\begin{enumerate}
\item If we set \(\lambda = (KT)^{\frac{1}{R}}\).   By Corollary~\ref{cor:ub},  for the case of two arms, Algorithm~\ref{alg:main} uses at most $R$ rounds and incurs an expected regret \(O\left((KT)^{\frac{1}{R}} \cdot \log(KT) \cdot \frac{1}{\Delta(I)} \right)\). Recall  by Theorem~\ref{thm:lb} that the expected regret needs to be $\Omega\left((KT)^{\frac{1}{R}} \cdot \frac{1}{\log(KT)\log\log(KT)} \cdot \frac{1}{\Delta(I)}\right)$.  For $R = O(1)$, which is of practical interest, our upper and lower bounds match up to a term that is logarithmic of $(KT)^{\frac{1}{R}}$.

\item  If we set $\lambda = \Theta(1)$, by Corollary~\ref{cor:ub}, 
Algorithm~\ref{alg:main} uses $O(\log (KT))$ rounds and achieves asymptotically optimal regret \(O\left(\sum_{a \neq \star} \frac{\log(KT)}{\Delta_a}\right)\). While the best centralized algorithm has essentially the same regret  \(O\left(\sum_{a \neq \star} \frac{\log(T)}{\Delta_a}\right)\)~\cite{GMS16}; recall that time $T$ in the centralized model corresponds to $KT$ in the CL model.
\end{enumerate}

\section{Concluding Remarks}
\label{sec:conclusion}
In this paper, we present the first set of round-regret tradeoffs for regret minimization in multi-armed bandits in the collaborative learning model. To the best of our knowledge, our lower bound results are the first to address the local adaptivity of agents for regret minimization in the collaborative learning model.  

We observe that a poly-logarithmic factor gap remains between our upper and lower bounds, potentially to be bridged in the future work.  It would also be interesting to generalize the results to non-IID environments, and investigate the round-regret tradeoffs for other bandits and reinforcement learning problems in the collaborative learning model.

\section*{Acknowledgments}
Nikolai Karpov and Qin Zhang are supported in part by NSF
CCF-1844234 and CCF-2006591.

\newpage

\bibliography{paper}

% \newpage~\newpage

% \begin{appendixcontent}
\newpage~\newpage
% \section{Appendix}
% \label{sec:app}

\noindent\rule{0.5\textwidth}{1pt}
\begin{center}
	\textbf{\large Appendix for {\em Communication-Efficient Collaborative Regret Minimization in Multi-Armed Bandits}}
\end{center}
\noindent\rule{0.5\textwidth}{1pt}

\section{Mathematics Tools}
\label{sec:tool}

\begin{lemma}[Hoeffding’s inequality]\label{lem:hoeffding}
	Let \(X_1, \dotsc, X_n \in [0, 1]\) be independent random variables and 
	$X = \frac{1}{n} \sum_{i = 1}^n X_i$.	Then 
	\begin{equation*}
		\Pr[\abs{X - \bE[X]} > t] \le 2 \exp(-2t^2n)\,.
	\end{equation*}
\end{lemma}

\begin{lemma}[Azuma's inequality]\label{lem:azuma}
	If the sequence of random variables \(Z_0, \dotsc, Z_n\) form a supermartingale and
	\begin{equation*}
		\forall t \in [n] : \abs{Z_{t} - Z_{t - 1}} \leq d\,,
	\end{equation*}
	then 
	\begin{equation*}
		\Pr[Z_n - Z_0 \geq \epsilon] \le \exp\left(  {\frac{-\epsilon^2}{2d^2n}} \right)\,.
	\end{equation*}
\end{lemma}

\section{Properties of Hard Inputs}
\label{sec:input-property}

The next two lemmas give some properties of the inputs in $\cali$ introduced in Section~\ref{sec:hard-dist}.  Both lemmas concern the ratio
$\ln\frac{g_A(\gamma)}{g_B(\gamma)}$,
where $A$ and $B$ are two inputs in $\cali$ and $\gamma$ is a transcript.

The first lemma says that the difference of the ratios on two transcripts of consecutive time steps is small. 
\begin{lemma}
	\label{lem:absolute-diff}
	Fix any $\ell \in [L]$ and any pair of inputs $A, B \in \cali_\ell$.  For any transcript $\gamma_n = \gamma_{n-1} \circ (j_n, o_n) = ((j_1, o_1), \ldots, (j_n, o_n))$ with $g_A(\gamma_n), g_B(\gamma_n) > 0$, it holds that
	$$\abs{\ln\frac{g_A(\gamma_n)}{g_B(\gamma_n)} - \ln\frac{g_A(\gamma_{n -1})}{g_B(\gamma_{n- 1})}} \le \frac{5}{\beta^\ell}.$$
\end{lemma}

\begin{proof}
	By (\ref{eq:g-def}), we can write
	\begin{equation*}
		\ln\frac{g_A(\gamma_n)}{g_B(\gamma_n)} 
		= \ln\frac{g_A(\gamma_{n -1})}{g_B(\gamma_{n- 1})} + \ln\frac{\Pr[\Theta_A(j_n) = o_n]}{\Pr[\Theta_B(j_n) = o_n]}. \label{eq:a-0}
	\end{equation*}
	We thus only need to show 
	\begin{equation}
		\label{eq:a-1}
		\abs{\ln\frac{\Pr[\Theta_A(j_n) = o_n]}{\Pr[\Theta_B(j_n) = o_n]}} \le \frac{5}{\beta^\ell}.
	\end{equation}
	By the definition of $I_\ell^\sigma$, both $\Pr[\Theta_A(j_n) = o_n]$ and $\Pr[\Theta_B(j_n) = o_n]$ are in the range  $\left[\frac{1}{2} - \frac{1}{\beta^\ell},  \frac{1}{2} + \frac{1}{\beta^\ell} \right]$. We thus have
	\begin{equation}
		\label{eq:a-2}
		\ln\frac{1 - 2  \beta^{-\ell}}{1 + 2  \beta^{-\ell}} \le \ln\frac{\Pr[\Theta_A(j_n) = o_n]}{\Pr[\Theta_B(j_n) = o_n]} \le \ln\frac{1 + 2  \beta^{-\ell}}{1 - 2  \beta^{-\ell}}\ .
	\end{equation}
	Using the fact that for any $x \in \left[-\frac{1}{2}, \frac{1}{2}\right]$, $x - x^2 \le \ln(1+x) \le x$, we have
	\begin{eqnarray}
		\ln\frac{1 + 2  \beta^{-\ell}}{1 - 2  \beta^{-\ell}} &=& \ln(1 + 2  \beta^{-\ell}) - \ln(1 - 2  \beta^{-\ell}) \nonumber \\
		&\le& 2  \beta^{-\ell} + 2  \beta^{-\ell} + 4  \beta^{-2\ell} \nonumber \\
		&\le& 5  \beta^{-\ell}. 		\label{eq:a-3}
	\end{eqnarray}
	Similarly, we have
	\begin{eqnarray}
		\label{eq:a-4}
		\ln\frac{1 - 2  \beta^{-\ell}}{1 + 2  \beta^{-\ell}} &=& \ln(1 - 2  \beta^{-\ell}) - \ln(1 + 2  \beta^{-\ell}) \nonumber \\
		& \ge & -2  \beta^{-\ell} - 4  \beta^{-2\ell} - 2  \beta^{-\ell} \nonumber \\
		& \ge & -5  \beta^{-\ell}.
	\end{eqnarray}
	Plugging (\ref{eq:a-3}) and (\ref{eq:a-4}) to (\ref{eq:a-2}), we get
	(\ref{eq:a-1}).
\end{proof}

Let $\A$ be a single-agent  algorithm for MAB and $I$ be an input, we use $\Gamma \sim \A(I, n)$ to denote a random transcript $\Gamma$ generated by $\A$ which runs on input $I$ for $n$ time steps.  

\begin{lemma}
	\label{lem:expected-diff}
	Let $\A$ be any single-agent  algorithm for MAB. For any $\ell \in [L]$ and any inputs $A, B, I \in \cali_\ell$, consider the random transcript $\Gamma_n = \Gamma_{n-1} \circ (J_n, O_n)  \sim \A(I, n)$,
	\begin{equation*}
		\bE_{\Gamma_n}\left[\left.\ln\frac{g_A(\Gamma_n)}{g_B(\Gamma_n)} - \ln\frac{g_A(\Gamma_{n - 1})}{g_B(\Gamma_{n - 1})} \ \right|\ \Gamma_{n -1}\right] \le \frac{11}{\beta^{2\ell}}.
	\end{equation*}
\end{lemma}

\begin{proof}
	
	For any fixed transcript $\Gamma_{n-1} = \gamma_{n-1}$ and algorithm $\A$ , $J_n = j_n$ is a deterministic value. 
	
	Using (\ref{eq:a-0}), we only need to prove for any $j_n \in \{1, 2\}$,  it holds that
	\begin{equation*}
		\label{eq:b-1}
		\bE_{O_n} \left[\ln\frac{\Pr[\Theta_A(j_n) = O_n]}{\Pr[\Theta_B(j_n) = O_n]}\right] \le  \frac{11}{\beta^{2\ell}},
	\end{equation*}
	where $O_n$ is the (random) reward of pulling the $j_n$-th arm in $I$.
	
	Let $\delta_I, \delta_A, \delta_B$ be three values such that  $\Pr[\Theta_I(j_n) = 1] = 1/2+\delta_I$,  $\Pr[\Theta_A(j_n) = 1] = 1/2+\delta_A$, and $\Pr[\Theta_B(j_n) = 1] = 1/2+\delta_B$.
	By the property of inputs in $\cali$, we know that the absolute values of $\delta_I, \delta_A, \delta_B$ are at most ${1}/{\beta^\ell}$.
	
	We immediately have $\Pr[\Theta_I(j_n) = 0] = 1/2-\delta_I$, $\Pr[\Theta_A(j_n) = 0] = 1/2-\delta_A$, and $\Pr[\Theta_B(j_n) = 0] = 1/2-\delta_B$. We also have $\Pr[O_n = 1] = 1/2 + \delta_I$, and $\Pr[O_n = 0] = 1/2 - \delta_I$.  
	
	With the above notations, we write
	\begin{eqnarray}
		&&\bE_{O_n} \left[\ln\frac{\Pr[\Theta_A(j_n) = O_n]}{\Pr[\Theta_B(j_n) = O_n]}\right] \nonumber 
		\\ &&=  \left(\frac{1}{2} + \delta_I\right) \ln\frac{1+2\delta_A}{1+2\delta_B} + \left(\frac{1}{2} - \delta_I\right) \ln\frac{1-2\delta_A}{1-2\delta_B} \nonumber \\
		 &&=  \frac{1}{2} \ln\frac{1 - 4\delta_A^2}{1 - 4\delta_B^2} + \delta_I \ln\frac{(1+2\delta_A)(1-2\delta_B)}{(1+2\delta_B)(1-2\delta_A)} \label{eq:b-2}.
	\end{eqnarray}	
	We bound the two terms in (\ref{eq:b-2}) separately.  For the first term,
	\begin{eqnarray}
		\ln\frac{1 - 4\delta_A^2}{1 - 4\delta_B^2} & = & \ln(1 - 4\delta_A^2) - \ln(1 - 4\delta_B^2) \nonumber \\
		&\le& -4\delta_A^2 + 4\delta_B^2 + 16 \delta_B^4.
		\label{eq:b-3}
	\end{eqnarray}
	For the second term, 
	\begin{eqnarray}
		&& \ln\frac{(1 + 2\delta_A)(1 - 2\delta_B)}{(1 + 2\delta_B)(1 - 2\delta_A)} 
		  \nonumber \\ &&\le  2\delta_A - (-2\delta_A - 4\delta_A^2) - 2\delta_B -  (2\delta_B - 4\delta_B^2) \nonumber \\
		&&= 4\delta_A + 4\delta_A^2 - 4\delta_B + 4\delta_B^2.
		\label{eq:b-4}
	\end{eqnarray}
	Plugging (\ref{eq:b-3}) and (\ref{eq:b-4}) to (\ref{eq:b-2}), we have
	\begin{eqnarray}
		&&\bE_{O_n} \left[\ln\frac{\Pr[\Theta_A(j_n) = O_n]}{\Pr[\Theta_B(j_n) = O_n]}\right]  
		\nonumber \\ &&\le  (-2\delta_A^2 + 2\delta_B^2 + 8 \delta_B^4) + \delta_I (4\delta_A + 4\delta_A^2 - 4\delta_B + 4\delta_B^2)  \nonumber \\
		&&\le \frac{11}{\beta^{2\ell}},   
		\label{eq:b-5}
	\end{eqnarray}
	where the last inequality is due to $\abs{\delta_A}, \abs{\delta_B}, \abs{\delta_I}  \le \frac{1}{\beta^\ell}$.
\end{proof}

\section{Proof of Lemma~\ref{lem:event-E}}
\label{app:proof-lem-event-E}

\begin{proof}
	By a union bound, we have
	\begin{equation}
		\label{eq:f-1}
		\Pr_{\Gamma \sim \A(I, n)}[\bar{\E}] \le  \sum_{\substack{\ell \in [L]: \\ \frac{\lambda  \beta^{2\ell}}{ \log L}  \ge n}} \sum_{A, B \in \cali_\ell} \Pr_{\Gamma \sim \A(I, n)} \left[\ln\frac{g_A(\Gamma)}{g_B(\Gamma)} > 2\eps\right].
	\end{equation}
	We try to bound each term in the summation of (\ref{eq:f-1}). Consider a fixed $\ell$ satisfying $\frac{\lambda  \beta^{2\ell}}{\log L} \ge n$ and a fixed pair of inputs $A, B \in \cali_\ell$.  Let $\Gamma_n \sim \A(I, n)$, and $\Gamma_t\ (t \in [n])$ be the prefix of $\Gamma_n$ of length $t$.  We introduce the following sequence of random variables for $t = 1, \ldots, n$:
	\begin{equation*}
		\ Z_t \triangleq \ln\frac{g_A(\Gamma_t)}{g_B(\Gamma_t)} - \frac{11}{\beta^{2\ell}} t\ .
	\end{equation*}
	We also define $Z_0 \triangleq 0$.
	\begin{claim} 
		\label{cla:supermartingale}
		$Z_0, Z_1, \ldots, Z_n$ form a supermartingale. 
	\end{claim}
	
	\begin{proof}
		We write
		\begin{eqnarray*}
			% & &	\bE_{\Gamma_t} [Z_t - Z_{t-1}\ |\ Z_{t-1}] \\
			&&\bE_{\Gamma_t} [Z_t - Z_{t-1}\ |\ \Gamma_{t-1}] 
			 \\  &&=   \bE_{\Gamma_t} \left[\left.\ln\frac{g_A(\Gamma_t)}{g_B(\Gamma_t)}  - \ln\frac{g_A(\Gamma_{t-1})}{g_B(\Gamma_{t-1})}  \ \right|\ \Gamma_{t-1} \right] - \frac{11}{\beta^{2\ell}}  \\ 
		 && \le 0,
		\end{eqnarray*}
		where the last inequality follows from Lemma~\ref{lem:expected-diff}.
		Combining with the tower rule we get
		\begin{eqnarray*}
			 &&\bE_{\Gamma_t} [Z_t - Z_{t - 1} \mid Z_{t - 1}]  \\ 
			&& =  \bE_{\Gamma_t}[\bE_{\Gamma_t}[Z_t - Z_{t - 1} \mid \Gamma_{t - 1}] \mid Z_{t - 1}] \\ && \le 0
		\end{eqnarray*}
	\end{proof}
	
	By Lemma~\ref{lem:absolute-diff}, we have
	\begin{equation*}
		\label{eq:f-2}
		\abs{Z_n - Z_{n-1}} \le \frac{5}{\beta^{\ell}} + \frac{11}{\beta^{2\ell}} \le \frac{10}{\beta^{\ell}}.
	\end{equation*}
	Applying Azuma's inequality (Lemma~\ref{lem:azuma}) on $Z_0, Z_1, \ldots, Z_n$ with $d = 10/\beta^{\ell}$, we get 
	\begin{eqnarray*}
		\Pr[Z_n - 0 \ge \eps] &\le& \exp\left(-\frac{\eps^2}{2 \cdot \left(10/\beta^{\ell}\right)^2 \cdot n}\right) \\
		&=& \exp\left(-\frac{\eps^2 \beta^{2\ell}}{200 n}\right) \\
		&=& \exp\left(-\frac{\eps^2 \log L}{200 \lambda}\right) \le \frac{1}{L^{10}}.  	
	\end{eqnarray*}
	%	\qinsays{$\frac{\eps^2}{200 \lambda} \ge 10$}
	
	Consequently, with probability $1 - 1/L^{10}$, we have
	\begin{eqnarray}
		\ln\frac{g_A(\Gamma_t)}{g_B(\Gamma_t)} &\le& Z_n + \frac{11}{\beta^{2\ell}}n \nonumber \\
		& < & \eps + \eps = 2\eps\ .  		\label{eq:f-3}
	\end{eqnarray}
	
	Combining (\ref{eq:f-1}) and (\ref{eq:f-3}), we have
	\begin{equation*}
		\Pr[\bar{\E}] \le L \cdot (2L)^2 \cdot 1/L^{10} \le 1/L^6.
	\end{equation*}
	The lemma follows.
\end{proof}

\section{Proof of Lemma~\ref{lem:indistinguishable}}
\label{app:proof-lem-indistinguishable}

\begin{proof}
	Let $\Gamma_A \sim \A(A, n)$ and $\Gamma_B \sim \A(B, n)$ be two random transcripts.  Based on our construction of hard inputs, it is easy to see that $\supp(\Gamma_A) = \supp(\Gamma_B)$. 
	Define a set of transcripts
	\begin{equation*}
		W \triangleq \left\{\gamma \mid (\gamma \in \supp(\Gamma_A)) \wedge \G(\gamma) \land \E(\gamma)\right\}\,.
	\end{equation*}
	By the law of total probability, we have
	\begin{equation*}
		\label{eq:g-1} 
		\Pr_{\Gamma \sim \A(A, n)}[\G(\Gamma) \wedge \E(\Gamma)] = \sum_{\gamma \in W} \Pr_{\Gamma \sim \A(A, n)}[\Gamma = \gamma], 
	\end{equation*}
	and 
	\begin{equation*}
		\label{eq:g-2}
		\Pr_{\Gamma \sim \A(B, n)}[\G(\Gamma) \wedge \E(\Gamma)] = \sum_{\gamma \in W} \Pr_{\Gamma \sim \A(B, n)}[\Gamma = \gamma].
	\end{equation*}
	Therefore, we only need to show
	\begin{equation}
		\label{eq:g-3} 
		\sum_{\gamma \in W} \Pr_{\Gamma \sim \A(A, n)}[\Gamma = \gamma] \le e^{2\eps} \sum_{\gamma \in W} \Pr_{\Gamma \sim \A(B, n)}[\Gamma = \gamma].
	\end{equation}
	
	By the definition of $g_I(\gamma)$ (Eq.\ (\ref{eq:g-def})), we have 
	\begin{equation}
		\label{eq:g-4} 
		\Pr_{\Gamma \sim \A(A, n)}[\Gamma = \gamma] = g_A(\gamma) \text{ and } \Pr_{\Gamma \sim \A(B, n)}[\Gamma = \gamma] = g_B(\gamma).
	\end{equation}
	
	When $\E(\gamma)$ holds,  by Definition~\ref{def:E},
	\begin{equation}
		\label{eq:g-5}
		g_A(\gamma) \le e^{2\eps} g_B(\gamma).
	\end{equation}
	Inequality (\ref{eq:g-3}) follows from (\ref{eq:g-4}) and (\ref{eq:g-5}).
\end{proof}

\section{Proof of Claim~\ref{cla:ell-star}}
\label{app:proof-cla-ell-star}

\begin{proof}
	By the definition of $\F_r(\gamma)$ and (\ref{eq:KT-L}), we have
	\begin{eqnarray}
		t_{{r(\gamma)}-1}  \le \frac{1}{K} \cdot (KT)^{\frac{r(\gamma)-1}{R}} 
		\le \frac{1}{K} \cdot \left(\frac{\beta^{2L}}{4}\right)^{\frac{R-1}{R}}. \label{eq:h-2}
	\end{eqnarray}
	By (\ref{eq:h-1}) and (\ref{eq:h-2}), we have 
	\begin{eqnarray*}
		\frac{\beta^{2(\ell(\gamma)-1)}}{\alpha K } \le \frac{1}{K} \cdot \left(\frac{\beta^{2L}}{4}\right)^{1-\frac{1}{R}},
	\end{eqnarray*}
	which implies 
	\begin{equation}
		\label{eq:h-9}
		\ell(\gamma) \le L - \frac{L}{R} + 1 + \frac{\log_\beta \alpha}{2}.
	\end{equation}
	When $\frac{L}{R} \ge \frac{\log_\beta \alpha}{2} + 1$ (recall the second inequality in (\ref{eq:R-L})), Inequality (\ref{eq:h-9}) implies
	$
	\ell(\gamma) \le L
	$.
	
	%\qinsays{We need $\frac{L}{R} \ge \frac{\log_\beta\alpha}{2} + 1$, or $R \le \frac{L}{\log_\beta\log L}$}
\end{proof}

\section{Proof of Claim~\ref{cla:large-mass}}
\label{app:cla-large-mass}

\begin{proof}
	By the definitions of projections of collaborative algorithms  and transcripts, we know that
	\begin{equation}
		\label{eq:l-1}
		\Pr_{\Gamma \sim \Pk^{\A_K}(I_L^+, \ell^*)}[\Gamma \in \Upsilon_k(\ell^*)] = \Pr_{\Gamma \sim \A_K(I_L^{+}, T)}[\Q(\Gamma)]  .
	\end{equation}
	
	Note that both $I_L^*$ and $I$ are in $\cali_\ell$.
	Using Lemma~\ref{lem:indistinguishable}, setting $A = I_L^+, B = I$, and $q = \abs{\Gamma} \le  \frac{\lambda  \beta^{2\ell^*}}{ \log L}$ (by (\ref{eq:k-2})), we have
	\begin{eqnarray}
		\label{eq:l-2}
		&&\Pr_{\Gamma \sim \Pk^{\A_K}(I_L^+, \ell^*)}[\Gamma \in \Upsilon_k(\ell^*)]  \nonumber  \\ &&\le e^{2\eps} \Pr_{\Gamma \sim \Pk^{\A_K}(I, \ell^*)}[\Gamma \in \Upsilon_k(\ell^*)] 
	\end{eqnarray}
	
	We thus have 
	\begin{eqnarray*}
		&&\Pr_{\Gamma \sim \Pk^{\A_K}(I, \ell^*)}[\Gamma \in \Upsilon_k(\ell^*) \wedge \E(\Gamma)] \\ 
		&&\stackrel{\text{Lemma}~\ref{lem:event-E}}{\ge} \Pr_{\Gamma \sim \Pk^{\A_K}(I, \ell^*)}[\Gamma \in \Upsilon_k(\ell^*)] - \frac{1}{L^6} \\
		&&\stackrel{(\ref{eq:l-2})}{\ge} e^{-2\eps} \Pr_{\Gamma \sim \Pk^{\A_K}(I_L^+, \ell^*)}[\Gamma \in \Upsilon_k(\ell^*)] - \frac{1}{L^6} \\
		&&\stackrel{(\ref{eq:l-1})}{=}   e^{-2\eps}  \Pr_{\Gamma \sim \A_K(I_L^{+}, T)}[\Q(\Gamma)] - \frac{1}{L^6} \\
		&&\stackrel{(\ref{eq:k-1})}{\ge} \frac{e^{-2\eps}}{L}   - \frac{1}{L^6} \\
		&&\ge \frac{e^{-2\eps}}{2L}.
	\end{eqnarray*}
	The claim follows.
\end{proof}

\section{Proof of Claim~\ref{cla:similar}}
\label{app:cla-similar}

\begin{proof}
	By the definitions of projections of collaborative algorithms and transcripts, we have
	\begin{equation*}
		\label{eq:-1}
		\Pr_{\Gamma \sim \A_K(I_L^{+}, T)}[\Pk(\Gamma, \ell^*) = \gamma]  =		\Pr_{\Gamma \sim \Pk^{\A_K}(I_L^+, \ell^*)}[\Gamma = \gamma] \ .
	\end{equation*}
	
	We thus only need to show that 
	$$\Pr_{\Gamma \sim \Pk^{\A_K}(I, \ell^*)}[\Gamma = \gamma] = c_\eps \Pr_{\Gamma \sim \Pk^{\A_K}(I_L^+, \ell^*)}[\Gamma = \gamma]$$
	for some $c_\eps \in [e^{-2\eps}, e^{2\eps}]$. It is easy to see that this equality is a direct consequence of Lemma~\ref{lem:indistinguishable} (note that $\abs{\Gamma} \le  \frac{\lambda  \beta^{2\ell^*}}{ \log L}$ due to (\ref{eq:k-2})).
\end{proof}

\section{Proof of Theorem~\ref{thm:ub}}
\label{app:ub-proof}

\paragraph{Correctness.}
Let $N = \abs{I}$ be the number of arms in the input $I$. We define the following event:
\begin{eqnarray*}
	&\E_2 :&  \forall a \in I, r \le \log_{\lambda}T, \ \ \abs{\hat{\mu}^{r}_a - \mu_a} \le \sqrt{\frac{\ln(T^3N)}{T_r}}\ .
\end{eqnarray*}

The following lemma states that 
\begin{lemma}
	$\Pr[\E_2] \ge 1 - 1/T^3$.
\end{lemma}

\begin{proof}
	By Hoeffding's inequality (Lemma~\ref{lem:hoeffding}), we have for any $a \in I$, for any $r \le \log_\lambda T$, it holds that
	\begin{eqnarray*}
		&& \Pr  \left[\abs{\hat{\mu}^{r}_a - \mu_a} > \sqrt{\frac{\ln (T^3 N)}{T_r}}\right] \nonumber \\ 
		&& \le  2 \exp\left(\frac{-2 \ln(T^3 N)}{T_r} \cdot T_r\right)\nonumber \\
		&& \le  \frac{2}{T^6 N^2} \,.
	\end{eqnarray*}
	By a union bound, we have
	\begin{eqnarray*}
		\Pr[\E_2] \ge 1 - N \cdot (1 + \log_\lambda T) \cdot \frac{2}{T^6 N^2} \ge 1 - \frac{1}{T^3}.
	\end{eqnarray*}
\end{proof}

\begin{lemma}
	\label{lem:star-survive}
	If $\E_2$ holds, then for any $r$ we have $\star \in I_r$.
\end{lemma}

\begin{proof}
	By the definition of event $\E_2$, we have for any $b \in I$,
	\begin{eqnarray*}
		\hat{\mu}^r_b - \hat{\mu}_{\star}^r & < & (\mu_b - \mu_\star)  + 2 \sqrt{\frac{\ln(T^3 N)}{T_r}} \le 2 \sqrt{\frac{\ln(T^3 N)}{T_r}}\,.
	\end{eqnarray*}
	Therefore, $\star \in I_r$ based on the description of Algorithm~\ref{alg:main}.
\end{proof}

For a suboptimal arm $a$, define $r(a)$ to be the smallest value such that $T_{r(a)} > \frac{64 \ln(T^3 N)}{\Delta_a^2}$.  The next lemma shows that all suboptimal arms will be eliminated by before the $r(a)$-th round ends.

\begin{lemma}
	\label{lem:elimination}
	If $\E_2$ holds, then any arm $a \neq \star$ does not appear in any $I_r$ for $r > r(a)$.	
\end{lemma}

\begin{proof}
	For any fixed arm $a \neq \star$, we consider the case that $a \in I_{r(a)}$, in which case the arm $a$ will be pulled for $T_{r(a)}$ times.  Abbreviating $r(a)$ as $r$, we have
	\begin{eqnarray*}
		 && \hat{\mu}^{r}_{\max} - \hat{\mu}^{r}_{a}  - 2 \sqrt{\frac{\ln(T^3 N)}{T_r}} \nonumber \\ 
		&&\ge \hat{\mu}^{r}_{\star} - \hat{\mu}^r_a - 2 \sqrt{\frac{\ln(T^3N)}{T_r}} \quad \text{(definition of $\hat{\mu}^{r}_{\max}$).}  \nonumber \\
		&&\ge  \mu_{\star} - \mu_{a} - 4\sqrt{\frac{\ln(T^3N)}{T_r}}  \quad \text{($\E_2$ holds)} \nonumber \\ 
		&&= \Delta_a - 4\sqrt{\frac{\ln(T^3 N)}{T_r}} \\
		&&\ge \Delta_a - 4 \cdot \frac{\Delta_a}{8}  \quad \text{(definition of $T_r$)} \\
		&&= \frac{\Delta_a}{2} > 0\ ,
	\end{eqnarray*} 
	which means that arm $a$ will be eliminated in the $r(a)$-th round.
\end{proof}

By Lemma~\ref{lem:elimination} and the fact that $T_{r+1} \le \lambda T_r$ for any $r$, the number of pulls on each suboptimal arm is bounded by 
\begin{eqnarray*}
	\max\{T_{r(a)}, \lambda \ln(T^3N)\} & \le & \frac{64 \lambda \ln(T^3 N)}{\Delta_a^2} + \lambda\ln(T^3N) \\ &\le& \frac{200 \lambda \ln(T N)}{\Delta_a^2}.
\end{eqnarray*}
Hence, the expected regret is bounded by 
\begin{equation}
	\sum_{a \neq \star} T_{r(a)} \Delta_a \le \sum_{a \neq \star} \frac{200 \lambda \ln(T N)}{\Delta_a} = O\left(\sum_{a \neq \star} \frac{\lambda \ln T}{\Delta_a} \right)\ ,
\end{equation}
where in the last equality, we have implicitly assumed that $T \ge N$, since each arm needs to be pulled at least once in order to identify the best arm.

\paragraph{Round Complexity.}  We next bound the number of batches. After the $r$-th round, the algorithm makes at least $\lambda^r$ pulls. Therefore, the number of rounds is upper bounded by $\log_\lambda T$ with certainty.  

On the other hand, if there is some round $r$ for which $\abs{I_r} = 1$, then we will pull this arm until the end of the time horizon.  By Lemma~\ref{lem:elimination}, all suboptimal arms will be pruned after $\max_{a \neq \star} r(a)$ rounds.  Therefore, if $\E_2$ holds, then the number of rounds can also be bounded by 
\begin{eqnarray*}
	&&\max\left\{\log_\lambda \frac{64\ln(T^3 N)}{(\Delta(I))^2} + 1 - \log_\lambda\ln(T^3N), 1\right\} \\ &&= 2\log_\lambda\frac{8}{\Delta(I)} + 1 \le 3 \log_\lambda\frac{8}{\Delta(I)}.\end{eqnarray*}

Therefore, the total number of rounds is also upper bounded by $3\log_\lambda \frac{8}{\Delta(I)}$ with probability $\left(1 - \frac{1}{T^3}\right)$.
% \end{appendixcontent}
\end{document}